%% file: main.tex
\documentclass{article}
\usepackage[utf8]{inputenc}
\usepackage{preamble}
\usepackage{amsmath}
\usepackage{bbm}
\usepackage{amsthm}
\usepackage{enumitem}
\usepackage{color}

\newcommand{\new}[1]{{\color{red} #1}}

\begin{document}

\title{Private Query Release Assisted by Public Data}
\author{Raef Bassily\thanks{Department of Computer Science \& Engineering, The Ohio State University.} \and Albert Cheu\thanks{Khoury College of Computer and Information Sciences, Northeastern University.} \and Shay Moran\thanks{Google AI Princeton} \and Aleksandar Nikolov\thanks{Department of Computer Science, University of Toronto} \and Jonathan Ullman\thanks{Khoury College of Computer and Information Sciences, Northeastern University.} \and Zhiwei Steven Wu\thanks{Department of Computer Science and Engineering, University of Minnesota.}}
\date{\today}
\maketitle

\begin{abstract}
We study the problem of differentially private query release assisted by access to public data. In this problem, the goal is to answer a large class $\cH$ of statistical queries with error no more than $\alpha$ using a combination of public and private samples. The algorithm is required to satisfy differential privacy only with respect to the private samples. We study the limits of this task in terms of the private and public sample complexities.

First, we show that we can solve the problem for any query class $\cH$ of finite VC-dimension using only $d/\alpha$ public samples and $\sqrt{p}d^{3/2}/\alpha^2$ private samples, where $d$ and $p$ are the VC-dimension and dual VC-dimension of $\cH$, respectively. In comparison, with only private samples, this problem cannot be solved even for simple query classes with VC-dimension one, and without any private samples, a larger public sample of size $d/\alpha^2$ is needed. Next, we give sample complexity lower bounds that exhibit tight dependence on $p$ and $\alpha$. For the class of decision stumps, we give a lower bound of $\sqrt{p}/\alpha$ on the private sample complexity whenever the public sample size is less than $1/\alpha^2$. Given our upper bounds, this shows that the dependence on $\sqrt{p}$ is necessary in the private sample complexity. We also give a lower bound of $1/\alpha$ on the public sample complexity for a broad family of query classes, which by our upper bound, is tight in $\alpha$.

\end{abstract}

\section{Introduction}

The ability to answer statistical queries on a sensitive data set in a privacy-preserving way is one of the most fundamental primitives in private data analysis. In particular, this task has been at the center of the literature of differential privacy since its emergence \cite{DN03, DMNS06, BLR08} and is central to the upcoming 2020 US Census release~\cite{DajaniLSKRMGDGKKLSSVA17}. In its basic form, the problem of differentially private query release can be described as follows. Given a class $\cH$ of queries $h:\cX\rightarrow \{\pm 1\}$\footnote{In this work, we focus on classes of binary functions (known in the literature of DP as \emph{counting queries}).} defined over some domain $\cX$, and a data set $\vec{x} = (x_1, \ldots , x_n)$ of i.i.d.\ drawn samples from some unknown distribution $\bD$ over $\cX$, the goal is to construct an $(\eps, \delta)$-differentially private algorithm that, given $\cH$ and $\vec{x}$, outputs a mapping $G: \cH \rightarrow [-1, 1]$ such that for every $h\in \cH$, $G(h)$ gives an accurate estimate for the true mean $\ex{x\sim \bD}{h(x)}$ up to some additive error $\alpha$. 

A central question in private query release is concerned with characterizing the \emph{private sample complexity}, which is the least amount of private samples required to perform this task up to some additive error $\alpha$. This question has been extensively studied in the literature on differential privacy \cite{DN03, HR10, MN12, BLR08, BUV14, SU15}. For general query classes, it was shown that the optimal bound on the private sample complexity in terms of $|\cX|$, $|\cH|$, and the privacy parameters is attained by the Private Multiplicative Weights (PMW) algorithm due to Hardt and Rothblum \cite{HR10}. This optimality was established by the lower bound due to Bun et al.~\cite{BUV14}. This result implied the impossibility of differentially private query release for certain classes of infinite size. Moreover, subsequent results by \cite{BNSV15, ALMM19} implied that this impossibility is true even for a simple class such as one-dimensional thresholds over $\mathbb{R}$. On the other hand, without any privacy constraints, the query release problem is equivalent to attaining uniform convergence over $\cH$, and hence the sample complexity is given by $d/\alpha^2,$ where $d$ is the VC-dimension of $\cH$.

In practice, it is often feasible to collect some amount of ``public'' data that poses no privacy concerns. For example, in the language of consumer privacy, there is considerable amount of data collected from the so-called ``opt-in'' users, who voluntarily offer or sell their data to companies or organizations. Such data is deemed by its original owner to pose no threat to personal privacy. There are also a variety of other sources of public data that can be harnessed. 

Motivated by the above observation, and by the limitations in the standard model of differentially private query release, in this work, we study a relaxed setting of this problem, which we call \emph{Public-data-Assisted Private (PAP)} query release. In this setting, the query-release algorithm has access to two types of samples from the unknown distribution: private samples that contain personal and sensitive information (as in the standard setting) and public samples. The goal is to design algorithms that can exploit as little public data as possible to achieve non-trivial savings in sample complexity over standard DP query-release algorithms, while still providing strong privacy guarantees for the private dataset.

\subsection{Our results}
In this work we study the private and public sample complexities of PAP query-release algorithms, and give upper and lower bounds on both. To describe our results, we will use $d$ and $p$ to denote the VC-dimension and dual VC-dimension of the query class $\cH$, respectively. We will use $\alpha$ to denote the target error for query release.

\begin{enumerate}[leftmargin=*] 
\item \textbf{Upper bounds:} We give a construction of a PAP algorithm that solves the query release problem for any query class $\cH$ using only $\approx d/\alpha$ public samples, and $\approx \sqrt{p}d^{3/2}/\alpha^2$ private samples. Recall that $d/\alpha^2$ samples are necessary even without privacy constraints; therefore, our upper bound on the public sample complexity shows a nearly quadratic saving. 

\item \textbf{Lower bound on private sample complexity:} We show that there is a query class $\cH$ with VC-dimension $d = \log(p)$ and dual VC-dimension $p$ such that any PAP algorithm either requires $\Omega(1/\alpha^2)$ public samples or requires $\tilde\Omega(\sqrt{p}/\alpha)$ total samples.  Thus the $\sqrt{p}$ dependence above is unavoidable.  For this class, $O(\log(p)/\alpha^2)$ public samples are enough to solve the problem with no private samples, and $\tilde{O}(\log(p)/\alpha^2 + \sqrt{p}/\alpha)$ private samples are enough to solve the problem with no public samples.  Thus, for this function class public samples essentially do not help improve the overall sample complexity, unless there are nearly enough public samples to solve the problem without private samples.

\item \textbf{Lower bound on public sample complexity:} We show that if the class $\cH$ has infinite {\it Littlestone dimension},\footnote{The Littlestone dimension is a combinatorial parameter that arises in online learning~\cite{littlestone1988learning,ben2009agnostic}.} then any PAP query-release algorithm for $\cH$ must have \emph{public} sample complexity~$\Omega(1/\alpha)$. The simplest example of such a class is the class of one-dimensional thresholds over $\mathbb{R}$. This class has VC-dimension $1$, and therefore demonstrates that the upper bound above is nearly tight.

\end{enumerate}

% a) Lower bound on private sample complexity when number of public samples $\lesssim 1/\alpha$

% \\

% b) Lower bound on public sample complexity for classes with infinite Littlestone dimension 

% \\

\subsection{Techniques} 

\paragraph{Upper bounds:} The first key step in our construction for the upper bound is to use the public data to construct a finite class $\tH$ that forms a ``good approximation'' of the original class $\cH$. Such approximation is captured via the notion of an $\alpha$-cover (Definition~\ref{defn:cover}). The number of public examples that suffices to construct such approximation
is about $d/\alpha$ \cite{ABM19}. Given this finite class $\tH$, we then reduce the original domain $\cX$ to a finite set $\cX_{\tH}$ of representative domain points, which is defined via an equivalence relation induced by $\cH$ over $\cX$ (Definition~\ref{defn:dual}). Using Sauer's Lemma, we can show that the size of such a representative domain is at most $\left(\frac{e\, \lvert \tH\rvert}{p}\right)^p$, where $p$ is the dual VC-dimension of $\cH$. At this point, we can reduce our problem to DP query-release for the finite query class $\tH$ and the finite domain $\cX_{\tH}$, which we can solve via the PMW algorithm \cite{HR10, DR14}.

\paragraph{Lower bound on private sample complexity:} 
The proof of the lower bound is based on \emph{robust tracing attacks}~\cite{BUV14,DSSUV15,SU15}.  That work proves privacy lower bounds for the class of \emph{decision stumps} over the domain $\{-1,1\}^p$, which contains queries of the form $h_j({x}) = x_j$ for some $j\in [p]$.  Specifically, they show that for any algorithm that takes at most $s \approx \sqrt{p}/\alpha$ samples, and releases the class of decision stumps with accuracy $\alpha$, there is some attacker that can ``detect'' the presence of at least $t \approx 1/\alpha^2$ of the samples.  Therefore, if the number of public samples is at most $t-1$, the attacker can detect the presence of one of the private samples, which means the algorithm cannot be differentially private with respect to the private samples.

\paragraph{Lower bound on public sample complexity:} This lower bound is derived in two steps. First, we show that PAP query-release for a class $\cH$ implies PAP learning (studied in \cite{BNS13, ABM19}\footnote{This notion was termed  ``semi-private'' learning in their work.}) of the same class with the same amount of public data. This step follows from a straightforward generalization of an analogous result by \cite{BNS13} in the standard DP model with no public data. Second, we invoke the lower bound of \cite{ABM19} on the public sample complexity of PAP learning. 

\subsection{Other related work}
To the best of our knowledge, our work is the first to formally study  differentially private query release assisted with public data. There has been work on private supervised-learning setting with access to limited public data, that is PAP learning. In particular, the notion of differentially private PAC learning assisted with public data was introduced by Beimel et al. in \cite{BNS13}, where it was called ``semi-private learning.'' They gave a construction of a learning algorithm in this setting, and derived upper bounds on the private and public sample complexities. The paper \cite{ABM19} revisited this problem and gave nearly optimal bounds on the private and public sample complexities in the agnostic PAC model. The work of \cite{ABM19} emphasizes the notion of $\alpha$-covers as a useful tool in the construction of such learning algorithms. Our PAP algorithm nicely leverages this notion in the query-release setting. 

In a similar vein, there has been work on other relaxations of private learning that do not require all parts of the data to be private. For example, \cite{chaudhuri2011sample, BNS13} studied the notion of ``label-private learning,'' where only the labels in the training set are considered private. Another line of work considers the setting in which the learning task can be reduced to privately answering classification queries \cite{hamm2016learning,papernot2017semi,papernot2018scalable,bassily2018NIPS,nandi2019privately}, where the goal is to construct a differentially private classification algorithm that predicts the labels of a sequence of public feature-vectors such that the predictions are differentially private in the private training dataset.

\input{prelims.tex}

\input{upper.tex}

\input{stumps-lb.tex}

\input{LB_public.tex}

\section*{Acknowledgments}  RB's research is supported by NSF Awards AF-1908281, SHF-1907715, Google Faculty Research Award, and OSU faculty start-up support. AC and JU were supported by NSF grants CCF-1718088, CCF-1750640, CNS-1816028, and CNS-1916020. AN is supported by an Ontario ERA, and an NSERC Discovery Grant RGPIN-2016-06333. ZSW is supported by a Google Faculty Research Award, a J.P. Morgan Faculty Award, and a Mozilla research grant. Part of this work was done while RB, AC, AN, JU, and ZSW were visiting the Simons Institute for the Theory of Computing.

%% Bibliography

    \bibliographystyle{alpha}

\bibliography{references}

\appendix

%%%% One-dimensional
%\input{one-dim.tex}

%\input{one-dim-LB.tex}

%%%% Two-dimensional
%\input{two-dim.tex}

%\input{two-dim-LB.tex}
\end{document}

%% file: prelims.tex
\section{Preliminaries}
  
We use $\cX$ to denote an arbitrary data universe, $\bD$ to denote a distribution over $\cX$, and $\cH \subseteq \{\pm 1\}^\cX$ to denote a binary \emph{hypothesis class}. 
%The statement ``$f$ is negligible'' is denoted by $f(n)=\negl(n)$. 

\subsection{Tools from learning theory}

\noindent The VC dimension of a binary hypothesis class $\cH\subseteq \{\pm 1\}^{\cX}$ is denoted by $\vc(\cH)$. 

%\noindent The \emph{expected disagreement} between a pair of hypotheses $H_1$ and $H_2$ with respect to a distribution $\cD$ over $\cX$ is defined as $\dis\left(H_1, H_2; ~\cD\right)\triangleq \ex{x\sim\cD_{\cX}}{\ind\left(h_1(x)\neq h_2(x)\right)}.$
\noindent We will use the following notion of coverings:

\begin{defn}[$\alpha$-cover for a hypothesis class]\label{defn:cover}
A family of hypotheses $\tH$ is said to form an $\alpha$-cover for a
hypothesis class $\cH\subseteq \{\pm 1\}^{\cX}$ with respect to a distribution $\bD$ over $\cX$ if for every $h\in \cH$, there is $\tlh\in\tH$ such that $\pr{x\sim \bD}{h(x)\neq \tlh(x)}\leq \alpha.$
\end{defn}

% \begin{defn}[The representative domain (or, \emph{the dual}) of a finite hypothesis class]\label{defn:dual}
% Let $\tH$ be a set of $M$ binary hypotheses $h_j:\cX\rightarrow \{\pm 1\},~j\in[M]$. The representative domain induced by $\tH$ on $\cX$, denoted by $\cX_{\tH}$, is a family of subsets of $\cX$ that forms an equivalence class for the set of all labeling vectors $\vec{y}\in\{\pm 1\}^{M}$ that can be generated by $\tH$. Namely, 
% $$\cX_{\tH}\triangleq \left\{\,\left\{x \in \cX: \left(h_1(x),\ldots, h_M(x)\right)=\vec{y}\,\right\}: ~\vec{y}\in \Pi_{\cX}(\tH)\,\right\},$$ 
% where $\Pi_{\cX}(\tH)=\left\{\left(h_1(x), \ldots, h_M(x)\right)\in \{\pm 1\}^{M}:~x\in \cX\right\}.$
% \end{defn}
%For example, consider the case where $\tH$ is a finite class of binary thresholds over $\R$: $h_{t_j}(x)=+1$ iff $x\leq t_j, ~j\in [M], ~t_1\leq t_2\leq \ldots \leq t_M$. Then, the representative domain in this case is a partition of $\R$ comprised of the $M+1$ intervals: $(\infty, t_1], (t_1, t_2], \ldots, (t_M, \infty).$ When $\tH$ is a class of $M$ halfspace functions in $\R^d$, then the representative domain is a partition of $\R^d$ comprised of all regions given by the intersection of $M$ halfspaces (the number of such regions is $O(M^d)$). 

\begin{defn}[Representative domain (a.k.a. the dual class)]\label{defn:dual}
Let $\tH\subseteq \{\pm 1\}^{\cX}$ be a hypothesis class. Define an equivalence relation on $\cX$ by $x\simeq x'$ if and only if $(\forall h\in\tH):h(x)=h(x')$. The representative domain induced by $\tH$ on $\cX$, denoted by $\cX_{\tH}$, is a complete set of distinct representatives from $\cX$ for this equivalence relation.
\end{defn}
For example, let $\tH$ be a class of $M$ binary thresholds over $\R$ given by: $h_{t_j}(x)=+1$ iff $x\leq t_j,~ j\in [M]$, $-\infty < t_1< t_2< \ldots < t_M <\infty$. Then, a representative domain in this case is a set of $M+1$ distinct elements; one from each of the following intervals  $(-\infty, t_1], (t_1, t_2], \ldots, (t_M, \infty).$ More generally, if $\tH$ is a class of halfspaces then a representative domain contains exactly one point in each cell of the hyperplane arrangement induced by $\tH$ (see Figure~\ref{fig:hyperplanes}). 

\begin{figure}
\begin{center}
    \includegraphics[scale=0.2]{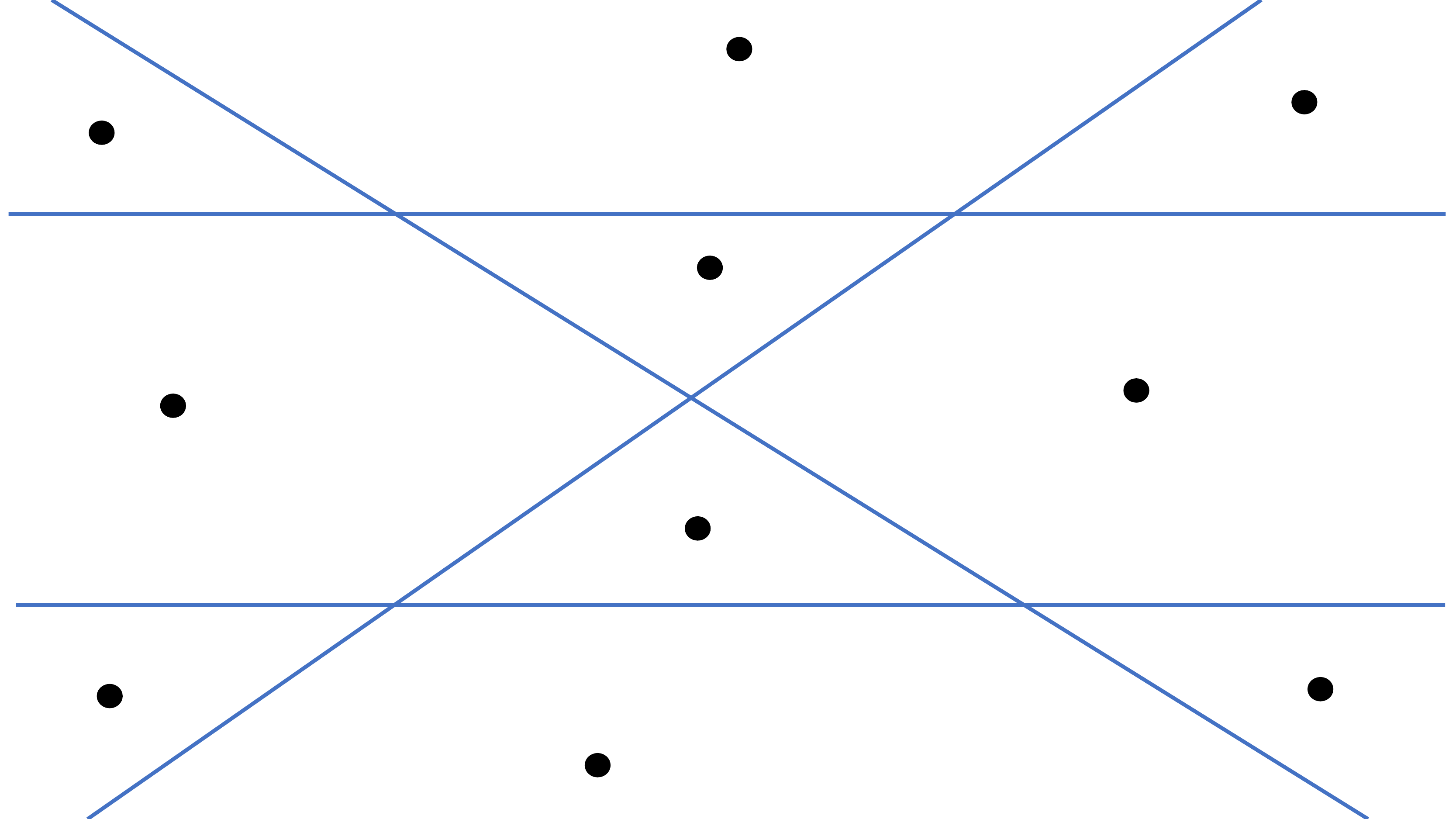}
\end{center}
\caption{A representative domain for a finite class of halfspaces}
\label{fig:hyperplanes}
\end{figure}

Note that when $\tH$ is finite then any representative domain for $\tH$ has size at most $2^{\lvert \tH\rvert}$,
since the equivalence class of each $x\in\cX$ is determined by the binary vector $(h(x))_{h\in \tH}$. Moreover, one can also make the following simple claim, which is a direct consequence of the Sauer-Shelah Lemma~\cite{sauer1972density} together with the fact that a representative domain $\cX_{\tH}$ has a one-to-one correspondence with the \emph{dual} class of $\tH$. Below, we use $\vcd(\cH)$ to denote the dual VC-dimension of a hypothesis class $\cH$, namely, $\vcd(\cH)$ is the VC-dimension of the dual class of $\cH$.
\begin{clm}\label{clm:size-rep-domain}
Let $\tH$ be a finite class of binary functions defined over a domain $\cX$. Then, the size of a representative domain $\cX_{\tH}$ satisfies:~ $\lvert\,\cX_{\tH}\,\rvert\leq \left(\frac{e\,\lvert\tH\rvert}{\vcd(\tH)}\right)^{\vcd(\tH)}$.
\end{clm}

The following useful fact gives a worst-case upper bound on the dual VC-dimension in terms of the VC-dimension.
\begin{fact}[\cite{Assouad83}]\label{fact:dualvc}
Let $\cH$ be a binary hypothesis class. The dual VC-dimension of $\cH$ satisfies: $\vcd(\cH)< 2^{\vc(\cH)+1}$. 
\end{fact}

\paragraph{Notation:} In Section~\ref{sec:upper}, we will use the following notation. Let $\tH$ be a hypothesis class defined over a domain $\cX$. For any $x\in \cX$, we define $\cX_{\tH}(x)$ as the representative $s\in\cX_{\tH}$ such that $x\simeq s$, where $\simeq$ is the equivalence relation described in Definition~\ref{defn:dual}. Note that by definition this $s\in \cX_{\tH}$ is unique. Moreover, for any $n$ and any $\vec{x}=(x_1, \ldots, x_n)\in \cX^n,$ we define $\cX_{\tH}(\vec{x})\triangleq \left(\cX_{\tH}(x_1), \ldots, \cX_{\tH}(x_n)\right)$.

%\paragraph{Notation:} In Section~\ref{sec:upper}, we will use the following notation. Let $\tH$ be a finite binary hypothesis class defined over a domain $\cX$. For any $s\in \cX_{\tH}$ and any $h\in\tH$, we will overload the notation and define $h(s)$ to be the label generated by $h$ on any $x\in s$. Note that, by definition, the label $h(x)$ is the same for all $x\in s$. For any $x\in \cX$, we define $\cX_{\tH}(x)$ as the set $s\in\cX_{\tH}$ such that $x\in s$. Note that also by definition such set $s$ is unique since $\cX_{\tH}$ is a partition of $\cX$. Hence, for any $n$ and any $\vec{x}=(x_1, \ldots, x_n)\in \cX^n,$ we define $\cX_{\tH}(\vec{x})\triangleq \left(\cX_{\tH}(x_1), \ldots, \cX_{\tH}(x_n)\right)$.

\begin{defn}[Query Release]
Given a distribution $\bD$ over $\cX$ and a binary hypothesis class $\cH$, a \emph{query release} data structure $G\in [-1,1]^\cH$ (equivalently, $G:\cH\rightarrow [-1, 1]$) estimates the expected label $\ex{x\sim\bD}{h(x)}$ for all $h\in\cH$. The worst-case error is defined as $$\Err_{\bD, \cH}(G) \triangleq \sup\limits_{h\in\cH}~ \lvert\, G(h)-\ex{x\sim\bD}{h(x)}\,\rvert.$$
\end{defn}

\subsection{Tools from Differential Privacy}
%{\color{red}define DP}
Two datasets $\vec{x},\vec{x}\,' \in \cX^n$ are neighboring when they differ on one element.
\begin{defn}[Differential Privacy]
A randomized algorithm $A:\cX^n \to \cZ$ is $(\eps,\delta)$-differentially private if for all neighboring $\vec{x}, \vec{x}\,'$ and all $Z \subseteq \cZ$
\[
\pr{}{A(\vec{x}) \in Z} \leq e^\eps \pr{}{A(\vec{x}\,') \in Z} + \delta
\]
\end{defn}

\paragraph{Private Multiplicative Weights (PMW):} In our construction in Section~\ref{sec:upper}, we will use, as a black box, a well-studied algorithm in differential privacy known as Private Multiplicative-Weights \cite{HR10}. We will use a special case of the offline version of the PMW algorithm. Namely, the input query class $\widetilde{\cH}$ is finite, and PMW runs over all the queries in the input class (in any order) to perform its updates, and finally outputs a query release data structure $\widetilde{G}\in [-1, 1]^{\widetilde{\cH}}$. When the input private data set $\vec{s}$ is drawn i.i.d. from some unknown distribution $\widetilde{\bD}$, the accuracy goal is to have a data structure $\widetilde{G}$ such that $\Err_{\widetilde{\bD}, \widetilde{\cH}}(\widetilde{G})$ is small. The outline (inputs and output of the PMW algorithm) is described in Algorithm~\ref{Alg:pmw}. 

\begin{algorithm}
\caption{An outline for the Private Multiplicative Weights Algorithm (PMW)}

\KwIn{Private data set $\vec{s}\in \widetilde{\cX}^n$ (where $\widetilde{\cX}$ is a finite domain); A finite query (hypothesis) class $\widetilde{\cH} \subseteq \{\pm 1\}^{\widetilde{\cX}}$; accuracy parameters $\alpha, \beta$; privacy parameters $\eps, \delta$.}
\KwOut{A data structure $\widetilde{G}:\widetilde{\cH}\rightarrow [-1, 1]$.}
\label{Alg:pmw}
\end{algorithm}

The following lemma is an immediate corollary of the accuracy guarantee of the PMW algorithm \cite{HR10,DR14}. In particular, it follows from combining the empirical accuracy of PMW with a standard uniform-convergence argument.

\begin{lem}[Corollary of Theorem 4.15 in \cite{DR14}]\label{lem:pmw}
For any $0 < \eps, \delta <1$, the PMW algorithm is $(\eps, \delta)$-differentially private. Let $\widetilde{\bD}$ be any distribution over $\widetilde{\cX}$. For any $0< \alpha, \beta <1$, given an input data set $\vec{s}\sim \widetilde{\bD}^n$ such that \ifnum\icml=0
$$n\geq \frac{200\cdot\sqrt{\log\left(\lvert\widetilde{\cX}\rvert\right)\,\log(2/\delta)}\,\left(\log\left(\lvert\widetilde{\cH}\rvert\right)+\log\left(\frac{128\,\log\left(\lvert\widetilde{\cX}\rvert\right)}{\alpha^2\beta}\right)\right)}{\eps\, \alpha^2},$$ \else \begin{align*}n\geq{}& \frac{200}{\eps\, \alpha^2} \cdot\sqrt{\log\left(\lvert\widetilde{\cX}\rvert\right)\,\log(2/\delta)}\,\cdot \\
&\left(\log\left(\lvert\widetilde{\cH}\rvert\right)+\log\left(\frac{128\,\log\left(\lvert\widetilde{\cX}\rvert\right)}{\alpha^2\beta}\right)\right), \end{align*}\fi
then, with probability at least $1-\beta$, PMW outputs a data structure $\widetilde{G}$ satisfying $\Err_{\widetilde{\bD}, \widetilde{\cH}}(\widetilde{G})\leq \alpha$.
\end{lem}

\subsection{Our Model: PAP Query Release}
In this paper, we study an extension of the problem of differentially private query release \cite{DR14} where the input data have two types: private and public. Formally, let $\bD$ be any distribution over a data domain $\cX$. Let $\cH\subseteq \{\pm 1\}^{\cX}$ be a class of binary queries. We consider a family of algorithms whose generic description (namely, inputs and outputs) is given by Algorithm~\ref{Alg:generic-QR}. 

\begin{algorithm}
\caption{An outline for a generic Public-data-Assisted Private (PAP) algorithm for query release}

\KwIn{Private data set $\vec{x} \in \cX^n$; public data set $\vec{w}\in \cX^m$; A query (hypothesis) class $\cH \subseteq \{\pm 1\}^{\cX}$; accuracy and confidence parameters $\alpha, \beta$; privacy parameters $\eps, \delta$.}
\KwOut{A data structure $G:\cH \rightarrow [-1, 1]$.}
\label{Alg:generic-QR}
\end{algorithm}

Given the query class $\cH$, a private data set $\vec{x}\sim\bD^n$ (i.e., a data set whose elements belong to $n$ private users), and a public data set $\vec{w}\in \cX^m$ (i.e., a data set whose elements belong to $m$ users with no privacy constraint), the algorithm outputs a query release data structure $G:\cH \rightarrow [-1, 1]$. %This data structure can be used to estimate $\ex{x\sim\bD}{h(x)}$ for all $h\in \cH$. 
Such an algorithm is required to be $(\eps, \delta)$-differentially private but only with respect to the private data set.  We call such an algorithm \emph{Public-data-Assisted Private (PAP) query-release} algorithm. 
The accuracy/utility of the algorithm is determined by the worst-case estimation error incurred by its output data structure $G$ on any query (hypothesis) $h\in \cH$.% Formally, we define the worst-case error of $G$ on $\cH$ with respect to the distribution $\bD$ as 
% \begin{align}
%     \Err_{\bD, \cH}(G)&\triangleq \sup\limits_{h\in\cH}~ \lvert\, G(h)-\ex{x\sim\bD}{h(x)}\,\rvert. \label{Err}
% \end{align}

\begin{defn}[$(\alpha, \beta, \eps, \delta)$ PAP query-release algorithm]\label{defn:pap-qr}
Let $\cH\subseteq \{\pm 1\}^{\cX}$ be a query class. Let $\cA: \cX^*\rightarrow [-1, 1]^{\cH}$ be a randomized algorithm in the family outlined in Algorithm~\ref{Alg:generic-QR}. We say that $\cA$ is $(\alpha, \beta, \eps, \delta)$ Public-data-Assisted Private (PAP) query-release algorithm for $\cH$ with private sample size $n$ and public sample size $m$ if the following conditions hold: 

\begin{enumerate}
    \item For every distribution $\bD$ over $\cX$,
    given data sets $\vec{x}\sim \bD^{n}$ and $\vec{w}\sim \bD^{m}$ as inputs to $\cA$, with probability at least $1-\beta$ (over the choice of $\vec{x}, \vec{w},$ and the random coins of $\cA$), $\cA$ outputs a function (data structure) $\cA\left(\vec{x}, \vec{w}\right)=G\in [-1, 1]^{\cH}$ satisfying~ $\Err_{\bD, \cH}(G)\leq \alpha.$ \label{cond:1}
    \item For all $\vec{w}\in\cX^{m},$ $\cA\left(\cdot, \vec{w}\right)$ is $(\eps, \delta)$-differentially private. \label{cond:2}
\end{enumerate}

\end{defn}

\begin{rem}\label{rem:proper_vs_improper_san}
In our description in Algorithm~\ref{Alg:generic-QR}, the algorithm is required to output a data structure $G:\cH\rightarrow [-1, 1]$ and not necessarily a ``synthetic'' data set $\vec{v}\in \cX^{n'}$ for some number $n'$ as in what is referred to as ``private proper sanitizers'' in \cite{BNS13}. In that special case, obviously the output data set can be used to define a data structure $G'$; namely, for any $h\in\cH$, $G'(h)\triangleq \frac{1}{n'}\sum_{i\in[n']}h(v_i)$. Moreover, in the general case, ignoring computational complexity, the output data structure can also be used to construct a data set as pointed out in \ifnum\icml=0 \cite[Remark~2.18]{BNS13} \else Remark 2.18 of \cite{BNS13} \fi. In particular, given a data structure $G$ whose error $\leq \alpha$, then it suffices find a data set $\vec{v}\in\cX^{n'}$, where $n'> \vc(\cH)/\alpha^2$, such that $\lvert \frac{1}{n'}\sum_{i\in [n']}h(v_i))-G(h)\rvert \leq 2\alpha$ for all $h\in \cH$, and hence the accuracy requirement would follow by the triangle inequality. Also, we know that this data set must exist. This is because by a standard uniform-convergence argument, a data set $\vec{s}\sim \bD^{n'}$ will, with a non-zero probability, satisfy $\lvert \frac{1}{n'}\sum_{i\in [n']}h(s_i))-\ex{x\sim\bD}h(x)\rvert \leq \alpha$ for all $h\in \cH$, and hence, by the triangle inequality, $\lvert \frac{1}{n'}\sum_{i\in [n']}h(s_i))-G(h)\rvert \leq 2\alpha$ for all $h\in \cH$.
\end{rem}

%% file: upper.tex
\section{A PAP Query-Release Algorithm for Classes of Finite VC-Dimension}\label{sec:upper}

We now describe a construction of a public-data-assisted private query release algorithm that works for any class with a finite VC-dimension. 

Our construction is given by Algorithm~\ref{Alg:prvQ}. The key idea of the construction is to use the public data to create a finite $\alpha$-cover $\tH$ for the input query class $\cH$ (see Definition~\ref{defn:cover}), then, run the PMW algorithm on the finite cover and the representative domain $\cX_{\tH}$ given by the dual of $\tH$ (see Definition~\ref{defn:dual}).

\begin{algorithm}
\caption{$\prvQ$ Public-data-assisted Private Query-Release Algorithm}

\KwIn{Private data set $\vec{x}=(x_1, \ldots, x_n) \in \cX^n$; public data set $\vec{w}=(w_1, \ldots, w_m) \in \cX^m$; A query  class $\cH \subseteq \{\pm 1\}^{\cX}$; accuracy and confidence parameters $\alpha, \beta$; privacy parameters $\eps, \delta$.}
\KwOut{A data structure $G:\cH \rightarrow [-1, 1]$.}

\medskip

\tcc{Use public data to construct $\alpha$-cover for $\cH$}

Let $T=\{\hw_1, \ldots,\hw_{\hatm}\}$ be the set of points in $\cX$ appearing at least once in $\vec{w}$.

Let $\Pi_{\cH}(T)=\left\{\left(h(\hw_1), \ldots, h(\hw_{\hatm})\right):~h\in\cH\right\}.$

Initialize $\tH=\emptyset$.

\For{each $\vec{y}=(y_1, \ldots, y_{\hatm})\in\Pi_{\cH}(T)$}{
Add to $\tH$ one arbitrary $h\in\cH$ that satisfies $h(\hw_j)=y_j$ for every $j=1,\ldots, \hatm$. 
} \label{step:rep-hyp}

\medskip

Let $\cX_{\tH}$ be a representative domain induced by $\tH$ on $\cX$ (as in Definition~\ref{defn:dual}).

\tcc{replace each point in the private data set $\vec{x}$ with its representative in $\cX_{\tH}$}

$\vec{s}\gets \cX_{\tH}(\vec{x})$

\medskip

\tcc{Run PMW algorithm over the data-set of representatives  $\vec{s} \in \cX^n_{\tH}$ and $\tH$}

$\widetilde{G}\gets \mathrm{PMW}\left(\, \vec{s},~ \tH, ~ \alpha/2, ~\beta/2, ~ \eps, ~\delta\right).$ 

\Return{ $G=G\left(\vec{w}\,,~ \tH\,,~ \widetilde{G}\,,~ \cdot\right)$}  \quad \tcc{see code below}

\medskip

\begin{center}
    ////////////////////////////////////////
\end{center} 

\tcc{Construct a function $G:\cH\rightarrow [-1, 1]$ as follows:}

\medskip

\textbf{Function}~ $G=G\left(\vec{w}\,,~ \tH\,,~ \widetilde{G}\,,~ \cdot\right)$

\medskip

\textbf{Input:} A query (hypothesis) $h\in \cH$.

\textbf{Output:} An estimate $r\in [-1, 1]$.

\medskip

$\tlh \gets \mathrm{Project}_{\tH, \vec{w}}(h)$, 

where $\mathrm{Project}_{\tH, \vec{w}}(h)$ denotes the unique $\tlh\in\tH$ s.t. $\left(\tlh(w_1),\ldots, \tlh(w_m)\right)=\big(\,h(w_1),\ldots, h(w_m)\,\big)$
\medskip

$r\gets \widetilde{G}(\tlh)$

\Return{$r$}

\medskip

\label{Alg:prvQ}
\end{algorithm}

\begin{thm}[Upper Bound]\label{thm:upper}
$\prvQ$ (Algorithm~\ref{Alg:prvQ}) is an $(\alpha, \beta, \eps, \delta)$ public-data-assisted private query-release algorithm for $\cH$ whose private and public sample complexities satisfy:
\begin{align*}
    n &= O\left(\frac{\left(d\,\log(1/\alpha)+\log(1/\beta)\right)^{3/2}\sqrt{p\,\log(1/\delta)}}{\eps\,\alpha^2}\right),\\
    m &= O\left(\frac{d\,\log(1/\alpha)+\log(1/\beta)}{\alpha}\right),
\end{align*}
where $d=\vc(\cH)$ and $p=\vc^{\bot}(\cH)$.
\end{thm}

\paragraph{Remark:} By Fact~\ref{fact:dualvc}, we can further bound the private sample complexity for general query classes as 
$$n = O\left(\frac{\left(d\,\log(1/\alpha)+\log(1/\beta)\right)^{3/2}\sqrt{2^{d}\,\log(1/\delta)}}{\eps\,\alpha^2}\right).$$

In the proof of Theorem~\ref{thm:upper},  we use the following lemma from \cite{ABM19}.

\begin{lem}[Lemma~3.3 in \cite{ABM19}]\label{lem:cover}
Let $\vec{w}\sim\bD^{m}$. Then, with probability at least $1-\beta/2,$ the family $\tH$ constructed in Step~\ref{step:rep-hyp} of Algorithm~\ref{Alg:prvQ} is an $\alpha/4$-cover for $\cH$ w.r.t.\ $\bD.$ In particular, for every $h\in\cH$, we have $\pr{x\sim\bD}{h(x)\neq \tlh(x)}\leq \alpha/4,$ where $\tlh=\mathrm{Project}_{\tH, \vec{w}}(h)$ (see Algorithm~\ref{Alg:prvQ} for the definition of $\mathrm{Project}_{\tH, \vec{w}}$\,), as long as 
$$m=\Omega\left(\frac{d\,\log(1/\alpha)+\log(1/\beta)}{\alpha}\right)$$
\end{lem}

\begin{proof}[Proof of Theorem~\ref{thm:upper}]
First, note that for any realization of the public data set $\vec{w},$ $\prvQ$ is $(\eps, \delta)$-differentially private w.r.t.\ the private data set. Indeed, the private data set $\vec{x}$ is only used to construct $\vec{s}=\cX_{\tH}(\vec{x})$, which is the input data set to the PMW algorithm. The output of PMW is then used to construct the output data structure $G$. Moreover, for any pair of neighboring data sets $\vec{x}, \vec{x}\,'$, the pair $\cX_{\tH}(\vec{x}), ~\cX_{\tH}(\vec{x}\,')$ cannot differ in more than one element. Hence, $(\eps, \delta)$-differential privacy of our construction follows from $(\eps, \delta)$-differential privacy of the PMW algorithm together with the fact that differential privacy is closed under post-processing. 

Next, we prove the accuracy guarantee of our construction. By Lemma~\ref{lem:cover}, it follows that with probability at least $1-\beta/2$, we have $\sup\limits_{h\in\cH}~\Big\vert\, \ex{x\sim\bD}{h(x)}-\ex{x\sim\bD}{\tlh(x)}\,\Big\vert \leq \alpha/2,$ where $\tlh=\mathrm{Project}_{\tH, \vec{w}}(h).$ Hence, it suffices to show that with probability at least $1-\beta/2$, $\Err_{\bD, \tH}(\widetilde{G})\leq \alpha/2$ (recall that $\widetilde{G}$ is the output data structure of the PMW algorithm). Note that by Sauer's lemma, we know $\lvert\, \tH \,\rvert\leq \left(\frac{e m}{d}\right)^d,$ where $d=\vc(\cH)$. From the setting of $m$ in the theorem statement, we hence have \ifnum\icml=0
$$\log\left(\lvert \tH\rvert\right)=O\left(d\log(1/\alpha)+d\left(\log\left(1+\frac{\log(1/\beta)}{d}\right)\right)\right)=O\left(d\,\log(1/\alpha)+\log(1/\beta)\right),$$ \else
\begin{align*}
\log\left(\lvert \tH\rvert\right)&=O\left(d\log(1/\alpha)+d\left(\log\left(1+\frac{\log(1/\beta)}{d}\right)\right)\right)\\
    &=O\left(d\,\log(1/\alpha)+\log(1/\beta)\right),
\end{align*}
\fi
Moreover, by Claim~\ref{clm:size-rep-domain}, we have \ifnum\icml=0
$$\log\left(\lvert\, \cX_{\tH} \,\rvert\right)=O\left(\vc^{\bot}(\tH)\, \big(d\,\log(1/\alpha)+\log(1/\beta)\big)\right)\leq O\left(p\, \big(d\,\log(1/\alpha)+\log(1/\beta)\big)\right),$$
\else
\begin{align*}
\log\left(\lvert\, \cX_{\tH} \,\rvert\right)&= O\left(\vc^{\bot}(\tH)\, \big(d\,\log(1/\alpha)+\log(1/\beta)\big)\right)\\
    &\leq O\left(p\, \big(d\,\log(1/\alpha)+\log(1/\beta)\big)\right),
\end{align*}
\fi
where $p=\vc^{\bot}(\cH)$.
Thus, given the setting of $n$ in the theorem statement,  Lemma~\ref{lem:pmw} implies that with probability at least $1-\beta/2$, our instantiation of the PMW algorithm yields $\widetilde{G}$ that satisfies 
\begin{align*}
  \alpha/2&\geq\sup\limits_{\tlh\in\tH}~\Big\vert\,\widetilde{G}(\tlh)- \ex{x\sim\bD}{\tlh\left(\cX_{\tH}(x)\right)}\Big
  \vert\\
  &=\sup\limits_{\tlh\in\tH}~\Big\vert\,\widetilde{G}(\tlh)- \ex{x\sim\bD}{\tlh(x)}\Big\vert\\
  &=\Err_{\bD, \tH}\left(\widetilde{G}\right),
\end{align*}
which completes the proof.
\end{proof}

%% file: stumps-lb.tex
\section{A Lower Bound for Releasing Decision Stumps} \label{sec:stumps-lb}

In this section we give an example of a hypothesis class---\emph{decision stumps} on over the domain $\cX = \{\pm 1\}^p$---where additional public data ``does not help'' for private query release.  This concept class can be released using either $\tilde{O}(\log(p)/\alpha^2 + \sqrt{p}/\alpha)$ private samples and no public samples, or using $O(\log(p)/\alpha^2)$ public samples and no private samples. However, we show that every PAP query-release algorithm requires either $\tilde{\Omega}(\sqrt{p}/\alpha)$ private samples or $\Omega(1/\alpha^2)$ public samples.  That is, making some samples public does not reduce the overall sample complexity until the number of the public samples is nearly enough to solve the problem on its own. 

The class of decision stumps on $\{\pm 1\}^p$ has dual-VC-dimension $p$, but VC-dimension just $\log p$, so this lower bound implies that the polynomial dependence on the dual-VC-dimension in Theorem~\ref{thm:upper} cannot be improved---there are classes with dual-VC-dimension $p$ that require either $\tilde{\Omega}(\sqrt{p}/\alpha)$ private samples or $\Omega(1/\alpha^2)$ public samples.

\begin{defn} [Binary Decision Stumps]
For any $p \in \N$, let $\cS_{p}$ be a hypothesis class of hypotheses $h : \pmo^{p} \to \pmo$ consisting of all hypotheses of the form $h_i(x) = x_i$ for $i \in [p]$.
\end{defn}

\begin{thm}[Lower Bound for Releasing Decision Stumps] \label{thm:stumps-lb}
Fix any $p \in \N$ and $\alpha > 0$.  Suppose $\cA$ is a PAP algorithm that takes $n$ private samples and $m$ public samples, satisfies $(1,1/4n)$-differential privacy, and is $(\alpha,\alpha)$-accurate for the class of decision stumps $\cS_{p}$.  Then either
$n= \tilde\Omega({\sqrt{p}}/{\alpha })$ 
or
$m = \Omega(1/\alpha^2)$.
\end{thm}

Thus, if $m=o(1/\alpha^2)$, then the number of private samples must scale proportionally to $\sqrt{p}$ as in our upper bound in Theorem~\ref{thm:upper}.

The main ingredient in the proof is a result of Dwork et al.~\cite{DSSUV15}.  Informally, what this theorem says is that for any algorithm that releases accurate answers to the class of decision stumps using too small of a dataset, there is an attacker who can identify a large number of that algorithm's samples.

\begin{thm}[Special Case of {\ifnum\icml=0\cite[Theorem 17]{DSSUV15}\else Theorem 17 of \cite{DSSUV15} \fi}]
\label{thm:trace}
For every $p \in \N$ and $\alpha > 0$, there exists a number 
$r = \tilde\Omega(\frac{\sqrt{p}}{\alpha})$
and a number
$t = \Omega(\frac{1}{\alpha^2})$
such that the following holds: 
For every query-release algorithm $\cA$ with total sample size $s\in [t, r+t]$ that is $(\alpha, \alpha)$-accurate for the class $\cS_{p}$ of decision stumps on $\pmo^p$, there exists a distribution $\widetilde{\bD}$ over $\pmo^{p}$ and an attacker $\cT$ who takes as input the vector of answers $q \in [-1,1]^p$ and an example $y \in \pmo^{p}$ and outputs either $\mathrm{IN}$ or $\mathrm{OUT}$ such that
\begin{align*}
&\pr{z_1,\dots,z_s,y \sim \widetilde{\bD} \atop q \sim \cA(z)}{\cT(q, y) = \mathrm{OUT}} \geq{} 1 - \tfrac{1}{(r+t)^2},~\textrm{and} \\
&\pr{z_1,\dots,z_s \sim \widetilde{\bD} \atop q \sim \cA(z)}{\left| \left\{ i \in [s] : \cT(q, z_i) = \mathrm{IN} \right\} \right| \geq \tfrac{t}{2}} \geq{} 1 - \tfrac{1}{(r+t)^2}.
\end{align*}
\end{thm}

\begin{proof}[Proof of Theorem~\ref{thm:stumps-lb}]
Fix $p, \alpha > 0$ and let $r$ and $t$ be the values specified in Theorem~\ref{thm:trace}.  Suppose that $\cA$ is a PAP algorithm that is $(\alpha,\alpha)$-accurate for the class $\cS_{p}$ with $n$ private samples and $m$ public samples. We will show that either $n > r$ or $m\geq t/2$. First, note that the accuracy condition of $\cA$ implies that we must have $n+m > t$ by the standard lower bound on the sample complexity of query release even without any privacy constraints. Thus, to prove the theorem statement, it suffices to show that if $m \leq \frac{t}{2}-1,$~ $t < n + m \leq r+t,$ and $\cA$ is $(\alpha,\alpha)$-accurate, then $\cA$ cannot satisfy $(1,1/4n)$-differential privacy w.r.t.\ its private samples. Indeed, this would imply that either $m\geq t/2$ ~or $n> t/2+r > r$.  

Let $\widetilde{\bD}$ be the distribution promised by Theorem~\ref{thm:trace}.  Let $\vec{x} = (x_1,\dots,x_n) \sim \widetilde{\bD}^n$ be a set of $n$ private samples and $\vec{w} = (w_1,\dots,w_m) \sim \widetilde{\bD}^m$ be a set of public samples, and let $\vec{z} = (x_1,\dots,x_n,w_1,\dots,w_m)$ be the combined set of samples.  Let $q \sim \cA(\vec{z})$.  By Theorem~\ref{thm:trace}, 

\begin{align*}
&\pr{}{\left| \left\{ i \in [n+m] : \cT(q, z_i) = \mathrm{IN} \right\} \right|\geq m+1}\\
\geq{}&\pr{}{\left| \left\{ i \in [n+m] : \cT(q, z_i) = \mathrm{IN} \right\} \right|\geq t/2}\\
\geq{}& 1-1/(r+t)^2\\
\geq{}& 1 - 1/(n+m)^2,    
\end{align*}
where the first inequality follows from the assumption that $m\leq \frac{t}{2}-1$, and the last inequality follows from the assumption that $n+m \leq r+t$.

That is, with high probability, the attacker identifies at least $m+1$ samples in the dataset.  Let $\mathbf{1}\left(\cT(q,z_i)=\mathrm{IN}\right)$ be the indicator of the event $\cT(q,z_i)=\mathrm{IN}$. Therefore, we have 
\ifnum\icml=0
$$\sum_{i=1}^{n} \pr{}{\cT(q,x_i) = \mathrm{IN}} + \sum_{i = 1}^{m} \pr{}{\cT(q,w_i) = \mathrm{IN}}
\geq{}  (m+1) \left(1 - 1/\left(n+m\right)^2\right)$$
\else
\begin{align*}
&\sum_{i=1}^{n} \pr{}{\cT(q,x_i) = \mathrm{IN}} + \sum_{i = 1}^{m} \pr{}{\cT(q,w_i) = \mathrm{IN}} \\
={}&\ex{}{\sum_{i=1}^{n} \mathbf{1}\left(\cT(q,x_i)=\mathrm{IN}\right)+\sum_{i = 1}^{m}\mathbf{1}\left(\cT(q,w_i)=\mathrm{IN}\right)}\\
\geq{}& (m+1)\cdot \pr{}{\left| \left\{ i \in [n+m] : \cT(q, z_i) = \mathrm{IN} \right\} \right|\geq m+1}\\
\geq{}&  (m+1)\left(1 - 1/\left(n+m\right)^2\right),    
\end{align*}
\fi
where the second step follows from Markov's inequality.
Since
$$
\sum_{i = 1}^{m} \pr{}{\cT(q,w_i) = \mathrm{IN}} \leq m
$$
we can conclude that
\begin{align*}
\sum_{i=1}^{n} \pr{}{\cT(q,x_i) = \mathrm{IN}} 
&\geq{} (m+1)(1 - 1/(n+m)^2) - m \\
&= 1 - (m-1)/(n+m)^2 \\
&\geq{} 1 - 1/(n+m) \\
&\geq{} 1 - 1/n \geq 1/2
\end{align*}
Therefore, there must exist a private sample $i^*$ such that
$$
\pr{x,w \sim \widetilde{\bD} \atop q \sim \cA(z)}{\cT(q,x_{i^*}) = \mathrm{IN}} \geq 1/2n
$$
Now, consider the dataset $\vec{z}_{\sim i^*}$ where we replace $x_{i^*}$ in $\vec{z}$ with an independent sample $y \sim \widetilde{\bD}$ but the rest of the samples in $\vec{z}_{\sim i^*}$ is the same as in $\vec{z}$.  In this experiment $x_{i^*}$ is now an independent sample from $\widetilde{\bD}$, so Theorem~\ref{thm:trace} states that
$$
\pr{\vec{x},\vec{w},y \sim \widetilde{\bD} \atop q \sim \cA(\vec{z}_{\sim i^*})}{\cT(q,x_{i^*}) = \mathrm{IN}} \leq 1/(n+m)^2 \leq 1/n^2
$$
However, note that the joint distribution $(\vec{z},\vec{z}_{\sim i^*})$ is a distribution over pairs of datasets that differ on at most one private sample.  Therefore, we have shown that $\cA$ cannot satisfy $(\eps, 1/4n)$-differential privacy for its private samples unless
\begin{align*}
&\frac{1}{2n} \leq e^{\eps} \cdot \frac{1}{n^2} + \frac{1}{4n} %\\
%\Longrightarrow &\frac{n}{4} \leq e^{\eps} \\
\quad \Longrightarrow \quad \ln(n/4) \leq \eps
\end{align*}
Therefore, in particular, for $n \geq 11$, $\cA$ cannot be $(1,1/4n)$-differentially private.
\end{proof}

%% file: LB_public.tex
\section{A Lower Bound on Public Sample Complexity}\label{sec:dichot}

The goal of this section is to show a general lower bound on the public sample complexity of PAP query release. Our lower bound holds for classes with infinite Littlestone dimension. The Littlestone dimension is a combinatorial parameter of hypothesis classes that characterizes mistake and regret bounds in Online Learning \cite{littlestone1988learning,ben2009agnostic}. There are many examples of classes that have finite VC-dimension, but infinite Littlestone dimension. The simplest example is the class of threshold functions over $\mathbb{R}$ whose VC-dimension is $1$, but has infinite Littlestone dimension. In \cite{ALMM19}, it was shown that if a class has infinite Littlestone dimension, then it is not privately learnable. 

Our lower bound is formally stated in the following theorem.

\begin{thm}[Lower bound on public sample complexity]\label{thm:lb_public}
Let $\cH\subseteq \{\pm 1\}^{\cX}$ be any query class that has infinite Littlestone dimension. Any PAP query-release algorithm for $\cH$ must have public sample complexity $m=\Omega(1/\alpha)$, where $\alpha$ is the desired accuracy. 
\end{thm}
We stress that the above lower bound on the public sample complexity holds regardless of the number of private samples, which can be arbitrarily large. 

In the light of our upper bound in Section~\ref{sec:upper}, our lower bound on the public sample complexity exhibits a tight dependence on the accuracy parameter $\alpha$. That is, one cannot hope to attain public sample complexity that is $o(
1/\alpha)$. 

In the proof of the above theorem, we will refer to the following notion of private PAC learning with access to public data that was defined in \cite{ABM19}. For completeness, we restate this definition here. 

\begin{defn}[$(\alpha, \beta, \eps, \delta)$ PAP Learner]\label{defn:pap-learner}
Let $\cH\subset \{\pm 1\}^{\cX}$ be a hypothesis class. 
A randomized algorithm $\cA$ is $(\alpha, \beta, \eps, \delta)$ PAP learner for $\cH$ with private sample size $n$ and public sample size $m$ if the following conditions hold: 

\begin{enumerate}
    \item For every distribution $\bD$ over $\mathcal{Z}=\cX\times \{\pm 1\}$,
    given data sets $\vec{x}\sim \bD^{n}$ and $\vec{w}\sim \bD^{m}$ as inputs to $\cA$, with probability at least $1-\beta$ (over the choice of $\vec{x}, ~\vec{w},$ and the random coins of $\cA$), $\cA$ outputs a hypothesis $\cA\left(\vec{x}, \vec{w}\right)=\hat{h}\in \{\pm 1\}^{\cX}$ satisfying 
$$\err_{\bD}\left(\hat{h}\right)\leq \inf\limits_{h\in\cH}\err_{\bD}\left(h\right)+\alpha,$$
where, for any hypothesis $h\in\{\pm 1\}^{\cX},$ $\err_{\bD}(h)\triangleq \pr{(x, y)\sim \bD}{h(x)\neq y}.$ %\label{cond:1}
    \item For all $\vec{w}\in\cZ^{m},$ $\cA\left(\cdot, \vec{w}\right)$ is $(\eps, \delta)$-differentially private. %\label{cond:2}
\end{enumerate}
We say that $\cA$ is proper PAP learner if $\cA\left(\vec{x}, \vec{w}\right)\in \cH$ with probability 1.
\end{defn}

\begin{proof}

We prove the above theorem in two simple steps that follow from prior works: the first step shows that PAP query-release implies PAP learning, and the second step invokes a known lower bound on PAP learning of classes with infinite Littlestone dimension. Both steps are formalized in the lemmas below.

\begin{lem}[General version of {\if\icml=0\cite[Theorem~5.5]{BNS13}\else Theorem 5.5 in \cite{BNS13} \fi}]\label{lem:release-to-learn}
Let $\cH\subseteq\{\pm 1\}^{\cX}$ be any class of binary functions. If there exists an $(\alpha, \beta, \eps, \delta)$ PAP query-release algorithm for $\cH$ with private sample complexity $n$ and public sample complexity $m$, then there exists an $(O(\alpha), O(\beta), O(\eps), O(\delta))$ PAP learner for $\cH$ with private sample complexity $n'=O(n\log(1/\alpha\beta)/\alpha^2\beta)$, and public sample complexity $m$.
\end{lem}

%{\color{gray} The second step of the proof is also straightforward. We simply invoke \cite[Theorem~4.1]{ABM19}, which we restate below:}

\begin{lem}[{\ifnum\icml=0\cite[Theorem~4.1]{ABM19}\else Theorem 4.1 in \cite{ABM19}\fi}]\label{lem:ABM-pub-lb}
Let $\cH$ be any class with an infinite Littlestone dimension (e.g., the class of thresholds over $\mathbb{R}$). 
Then, any PAP learner for $\cH$ must have public sample complexity $m = \Omega(1/\alpha)$, where $\alpha$ is the excess error.
\end{lem}

Given these two lemmas, the proof is straightforward. To elaborate, note that Lemma~\ref{lem:release-to-learn} shows that for any class $\cH$, a PAP query-release algorithm for $\cH$ with public sample complexity $m$ implies the existence of a PAP learner for $\cH$ with the \emph{same} public sample complexity (and essentially the same accuracy and privacy parameters). Hence, by Lemma~\ref{lem:ABM-pub-lb}, if $\cH$ has infinite Littlestone dimension, then such public sample complexity must satisfy $m=\Omega(1/\alpha)$. This proves our theorem. 

Although the proof of Lemma~\ref{lem:release-to-learn} is almost straightforward given \ifnum\icml=0\cite[Theorem~5.5]{BNS13}\else Theorem 5.5 in \cite{BNS13}\fi, we will elaborate on a couple of minor details. First, note that even though the reduction in \cite{BNS13} involves pure differentially private algorithms, the same construction in their reduction would also work for the case of $(\eps, \delta)$-differential privacy with minor and obvious changes in the privacy analysis. %Also, note that we are only concerned here about deriving a lower bound on the public sample complexity, and hence we did not show an explicit bound for the final private sample complexity (which appears in \cite{BNS13}) since it is irrelevant for our purpose. 
Second, we note that the reduction in \cite{BNS13} is for ``proper sanitizers,'' which are query-release algorithms that are restricted to output a data set from the input domain rather than any data structure that maps $\cH$ to $[-1, 1]$. As discussed in Remark~\ref{rem:proper_vs_improper_san}, ignoring computational complexity, any PAP query-release algorithm satisfying Definition~\ref{defn:pap-qr} can be transformed into a PAP query-release algorithm that outputs a data set from the input domain and has the same accuracy (up to a constant factor). Now, given these minor details and since any PAP algorithm can obviously be viewed as a differentially private algorithm operating on the private data set (by ``hardwiring'' the public data set into the algorithm), Lemma~\ref{lem:release-to-learn} simply follows by invoking the reduction in \cite{BNS13}. \end{proof}